\documentclass{article}

\usepackage[final]{corl_2020} 

\usepackage{amsthm}
\usepackage{amsfonts}       
\newtheorem{theorem}{Theorem}
\usepackage{caption}
\usepackage{subcaption}
\usepackage{graphicx}  
\usepackage{hyperref}
\usepackage{multirow}
\usepackage{array}
\newcolumntype{P}[1]{>{\centering\arraybackslash}p{#1}}
\usepackage[normalem]{ulem}
\definecolor{orange}{rgb}{1,0.5,0}
\usepackage{footnote}
\makesavenoteenv{tabular}
\makesavenoteenv{table}
\usepackage{amsmath}
\usepackage{amssymb}

\title{Learning from Suboptimal Demonstration via Self-Supervised Reward Regression}

\author{%
  Letian Chen, Rohan Paleja, Matthew Gombolay \\
  School of Interactive Computing \\
  Georgia Institute of Technology \\
  Atlanta, GA 30332 \\
  \texttt{\{letian.chen, rohan.paleja, matthew.gombolay\}@gatech.edu}\\
}

\begin{document}
\maketitle


\begin{abstract}
Learning from Demonstration (LfD) seeks to democratize robotics by enabling non-roboticist end-users to teach robots to perform a task by providing a human demonstration. However, modern LfD techniques, e.g. inverse reinforcement learning (IRL), assume users provide at least stochastically optimal demonstrations. This assumption fails to hold in most real-world scenarios. Recent attempts to learn from sub-optimal demonstration leverage pairwise rankings and following the Luce-Shepard rule. However, we show these approaches make incorrect assumptions and thus suffer from brittle, degraded performance. We overcome these limitations in developing a novel approach that bootstraps off suboptimal demonstrations to synthesize optimality-parameterized data to train an idealized reward function. We empirically validate we learn an idealized reward function with $\sim0.95$ correlation with ground-truth reward versus  $\sim 0.75$ for prior work. We can then train policies achieving $\sim 200\%$ improvement over the suboptimal demonstration and $\sim 90\%$ improvement over prior work. We present a physical demonstration of teaching a robot a topspin strike in table tennis that achieves $32\%$ faster returns and $40\%$ more topspin than user demonstration.
\end{abstract}



\section{Introduction}
\label{sec:introduction}
Recent advances in robot learning offer the promise of benefiting people across a variety of applications, ranging from healthcare~\cite{gombolay2018robotic,krishnan2017transition} to manufacturing~\citep{5756872,wangheterogeneous}, scientific discovery~\citep{sparkes2010towards}, and even assistive, household tasks~\citep{cakmak2013towards}. However, the traditional model for deploying robots in real-world settings typically consists of an army of expert consultants who design, build, and program robots for each application. This \emph{ad hoc} model is cost-intensive and stands in the way of realizing a vision for ubiquitous and democratized robots -- robots that can learn from and in the hands of real end-users.
The field of Learning from Demonstration (LfD) has sought to democratize access to robotics by empowering end-users to program robots by demonstrations rather than requiring users to program robots with a computing language~\citep{ravichandar2020recent}. 

LfD research can be broken down into two key variants: imitation learning (IL) and inverse reinforcement learning (IRL). IL~\citep{ross2011reduction} directly learns a predictive model, a mapping from a state to the desired action. Despite IL's simplicity, it is ill-suited to learn from suboptimal demonstrations as the goal is to imitate the behavior rather than discover the latent objective of the demonstrator. Inverse Reinforcement Learning (IRL)~\citep{abbeel2004apprenticeship,Chen2020JointGA}, on the other hand, aims to infer a task's underlying objective by learning a policy from one or more demonstrations of that task. With this reward function, one can leverage Reinforcement Learning (RL)~\citep{10.5555/3312046} to find a policy that can autonomously accomplish the desired goal. As such, we investigate a novel IRL-based method for learning from suboptimal demonstration that accurately infers the latent, idealized reward.

However, prior work in IRL relies on fundamental assumptions that fail to hold in all but the most isolated, controlled scenarios. For example, Maximum Margin IRL~\citep{abbeel2004apprenticeship} assumes the optimality of demonstrations and does not generally produce better performance when demonstrations are suboptimal. Probabilistic IRL approaches such as maximum-entropy IRL (\citep{ziebart2008maximum}) and Bayesian IRL (BIRL~\citep{ramachandran2007bayesian}) relax this assumption of absolute optimality to stochastic optimality or ``soft'' optimality. While it is possible for probabilistic IRL to recover the optimal reward and policy when the demonstrations are slightly suboptimal according to the stochastic optimality assumption, in general, it cannot produce a much better policy than the suboptimal demonstrations.

Preference-based Reinforcement Learning~\cite{wirth2017survey} (PbRL) allows for learning algorithms to take advantage of non-numerical feedback within their policy learning, avoiding the need to have a \textit{perfect} scoring metric for demonstrations. While PbRL has been widely successful for both RL~\citep{pbpl} and IRL~\citep{dmirl, pmlr-v24-valko12a,Sugiyama2012PreferencelearningBI}, D-REX~\citep{brown2020better} has proven to be the current state-of-the-art for learning from suboptimal demonstration. D-REX utilizes behavioral cloning (BC) and a ranking formulation based on the Luce-Shepard Choice Rule~\cite{luce2012individual} to infer an idealized reward function. However, we show in Section \ref{subsec:D-REX_assumption} that Luce-Shepard Rule results in a counterproductive inductive bias. Furthermore, we find in Section \ref{sec:results} that BC-based rollout generation for synthesizing ranked, suboptimal demonstrations results in a reward function brittle against covariate shift~\citep{ross2011reduction}. In this paper, we overcome the limitations of prior work by developing a novel algorithmic framework, Self-Supervised Reward Regression (SSRR), for learning from suboptimal demonstration. 
\paragraph{Contributions:} 
\begin{enumerate}\vspace{-8pt}
    \item We empirically characterize the relationship between a policy's performance and the amount of noise injected to synthesize optimality-ranked trajectories across ten test domains and three RL baselines. With this characterization, we show that prior work fails to capture this relationship, resulting in poor learning from suboptimal demonstration.\vspace{-4pt}
    \item We propose a novel IRL framework, SSRR, that accurately learns the latent reward function described by suboptimal demonstration by leveraging a low-pass filter based on our characterization from the first contribution. With this filter, we show that SSRR can learn more accurate reward functions from suboptimal demonstration, achieving $0.94$, $0.94$, and $0.97$ correlation with the ground-truth reward across three simulated robot control tasks (HalfCheetah, Hopper, and Ant) versus only $0.854$, $0.797$, and $0.621$\ for prior work~\citep{brown2020better}.\vspace{-4pt}
    \item We develop a powerful mechanism to synthesize trajectories, Noisy-AIRL, that enhances the robustness of our reward function. Combining SSRR and Noisy-AIRL, we show that we can train a policy with our reward function to achieve an average of $163\%$, $192\%$, and $141\%$ improvement over the suboptimal demonstration across three simulated robot control tasks versus only $49\%$, $89\%$, and $39\%$ for prior work.\vspace{-4pt}
    \item We provide a real-world demonstration of SSRR in learning from suboptimal demonstration to perform a topspin strike in table tennis. We show our policy trained on SSRR reward achieves $32\%$ faster return speeds and $40\%$ more topspin than the user demonstrations.
\end{enumerate}

\section{Preliminaries}
In this section, we introduce Markov Decision Processes, IRL, and D-REX~\cite{brown2019extrapolating}. 
\paragraph{Markov Decision Process --} A Markov Decision Process (MDP) $M$ is defined as a 6-tuple $\langle S,A,R,T,\gamma,\rho_0\rangle$. $S$ and $A$ are the state and action spaces, respectively. $R(s,a)$ is the reward received by an agent for executing action, $a$, in state, $s$. $T(s^\prime|s,a)$ is the probability of transitioning from state, $s$, to state, $s'$, when applying action, $a$. $\gamma\in [0,1]$ is the discount factor, prioritizing long- versus short-term reward. $\rho_0(s)$ is the initial state distribution. A policy, $\pi(a|s)$, gives the probability of an agent taking action, $a$, in state, $s$. The goal of RL is to find the optimal policy, 
$\pi^*=\arg\max_\pi \mathbb{E}_{\tau\sim\pi}\left[\sum_{t=0}^T{\gamma^tR(s_t,a_t)}\right]$ to maximize cumulative discounted reward, where $\tau=\langle s_0,a_0,\cdots,s_T,a_T\rangle$ is an agent's trajectory. $R(\tau)=\sum_{t} R(s_t,a_t)$ is the reward operator over a trajectory.
We consider the maximum entropy version~\citep{ziebart2010modeling}, $\pi^*=\arg\max_\pi\mathbb{E}_{\tau\sim \pi}[\sum_{t=0}^T{\gamma^tR(s_t,a_t)}+\alpha H(\pi(\cdot | s_t))]$, which adds an entropy bonus (i.e., $H(\cdot)$) to favor stochastic policies and encourage exploration during training.

\paragraph{Inverse Reinforcement Learning --}
The goal of IRL is to take as input an MDP sans reward function and set of demonstration trajectories $\mathcal{D}=\{\tau_1, \tau_2, \cdots, \tau_N\}$ and return a recovered reward function $R$. Our method is based on AIRL~\cite{fu2017learning}, which casts the IRL problem in a generative-adversarial framework. AIRL consists of a discriminator and a generator in the form of a learned policy. The discriminator, $D$, is given by $D_\theta(s,a)=\frac{e^{f_\theta(s,a)}}{e^{f_\theta(s,a)}+\pi(a|s)}$, where $f_{\theta}(s,a)$ is the reward function, parameterized by $\theta$, and $\pi(a|s)$ is the policy. The reward function is updated using gradient descent via the loss function in Equation \ref{eqn:airl_loss}. Minimizing this cross entropy loss allows the discriminator to distinguish expert trajectories from generator policy rollouts. The policy, $\pi$, is trained to imitate the expert by maximizing the pseudo-reward function given by $\hat{R}= f_{\theta}(s,a)$.
\par\nobreak{ \small \noindent
\begin{gather}
    L_D=-\mathbb{E}_{\tau\sim\mathcal{D}, (s,a)\sim\tau}[\log D(s,a)]-\mathbb{E}_{\tau\sim\pi, (s,a)\sim\tau}[\log (1-D(s,a))] \label{eqn:airl_loss}
\end{gather}}
\vspace{-3mm}
\subsection{Disturbance-based Reward Extrapolation}
Disturbance-based Reward Extrapolation (D-REX)~\citep{brown2020better} performs IRL based on a ranking information to learn from suboptimal demonstration. D-REX first learns a policy $\pi_{BC}$ via behavioral cloning (BC) from an initial dataset, $\mathcal{D}$, of suboptimal demonstrations. Next, D-REX adds noise to $\pi_{BC}$ to create noisy trajectories, as given by Equation \ref{eq:noise_injection}, where $U$ is the uniform distribution and $\eta$ is the proportion of noise injection.
\par\nobreak{ \small \noindent
\begin{align}
\label{eq:noise_injection}
    \tau_\eta\sim \pi_\eta(a|s) = \eta U(a) + (1 - \eta)\pi_{BC}(a|s). 
\end{align}}D-REX assumes that the higher the noise level is, the lower the trajectory should be ranked (denoted as $\tau_i\succ\tau_j$), as given by $\tau_i\prec\tau_j, R(\tau_{\eta_i}) < R(\tau_{\eta_j})\ \forall\ \eta_i > \eta_j$. Finally, D-REX learns via Equation \ref{eqn:D-REX} the reward function through supervised learning on the trajectories' rankings, using a pairwise ranking loss according to Luce-Shepard rule~\citep{luce2012individual}.
\par\nobreak{ \small \noindent
\begin{align}
\label{eqn:D-REX}
    L(\theta)=-\frac{1}{|\mathcal{P}|}\sum_{(i,j)\in \mathcal{P}}\log\frac{\exp\left(\sum_{s\in\tau_i}{R_\theta(s)}\right)}{\exp\left(\sum_{s\in\tau_i}{R_\theta(s)}\right)+\exp\left(\sum_{s\in\tau_j}{R_\theta(s)}\right)},
\end{align}}Here, $\mathcal{P}=\{(i,j):\tau_i\succ\tau_j\}=\{(i,j):\eta_i<\eta_j\}$. In Section \ref{sec:approach}, we show this assumption produces an unhelpful inductive bias. Our method rather infers the relationship between noisy trajectories and performance, empirically learning a metric to properly relate demonstrations by their degree of suboptimality. In doing so, our approach sets a new state-of-the-art in learning from suboptimal demonstration, outperforming D-REX by $87.5\%$ on average across 3 different robotics tasks.


\section{Characterizing the Noise-Performance Curve for Reward Regression}
\label{sec:approach}
In this section, we inspect the inductive bias of D-REX (i.e., Luce-Shepard rule) and show empirically that it is a counter-productive assumption. Then, we propose an alternative formulation that accurately characterizes the relationship between a policy's performance and the injected noise.

\subsection{D-REX's Homogeneous Assumption on Noise-Performance Relationship}
\label{subsec:D-REX_assumption}
D-REX's loss function is based on the Luce-Shepard rule, which models decisions among \textit{discrete} and \textit{different} options. \citet{bobu2020less} points out that Luce's choice axiom is not suitable for ranking robot trajectories, which are \textit{continuous} and often are \textit{similar}. 
Moreover, D-REX's application of the Luce's choice rule does not accurately reflect the noise-performance degradation relationship. D-REX's loss function (Equation \ref{eqn:D-REX}) only depends on the number of noise levels $\eta$, and treats interval data as merely ordinal. We derive a theorem to illustrate this flaw in the case of learning from three demonstrations of differing optimality without loss of generality. 

\begin{theorem}
\label{theorem:3_noise}
Consider three trajectories associated with three noise levels, $(\eta_1,\tau_1)$, $(\eta_2,\tau_2)$, and $(\eta_3,\tau_3)$, with $\eta_1<\eta_2<\eta_3$. Denote $r_i$ as the cumulative reward for $\tau_i$, i.e., $R_\theta(\tau_i)$. Optimizing D-REX's loss function (Equation \ref{eqn:D-REX}) results in $r_2=\frac{r_1+r_3}{2}$. Proof is attached in the supplementary. 
\end{theorem}
The relationship $r_2=\frac{r_1+r_3}{2}$ is generally not true, as we show empirically in Section \ref{subsec:performance_noise}. Further, $r_2$ is generally a function of $\eta_2$, which is ignored by D-REX.
In an extreme example, a set of demonstrations with injected noise levels of $0\%, 50\%, 100\%$ would have the same reward relationship as demonstrations with injected noise levels $0\%, 1\%, 100\%$, resulting in the same reward, $r_2$, for $50\%$ noise and $1\%$ noise. This contradiction motivates us to utilize environment interaction to learn a more accurate noise-performance relationship and, in turn, a better reward function.

\subsection{The Empirical Noise-Performance Relationship}
\label{subsec:performance_noise}
We empirically investigate the noise-performance relationship curve for a diverse set of five Atari games and five MuJoCo control tasks~\citep{todorov2012mujoco}. For each domain, we learn an optimal policy, $\pi^*$, and inject noise at various levels (Equation \ref{eq:noise_injection}). We show the noise-performance relationship of a trained PPO~\citep{schulman2017proximal} agent in Figures \ref{fig:atari_performance_noise} and \ref{fig:mujoco_performance_noise}. Results of DQN~\citep{mnih2015human} an SAC~\citep{haarnoja2018soft} agents are in the supplementary material. We find that the relationship between the performance of the policy and the amount of injected noise depends on the interval -- not just ordinal -- value of the noise and the specific environment.
On the other hand, D-REX's approach incorrectly assumes a homogeneous and ordinal noise-performance relationship across environments and policies.

\begin{figure}[tb]
    \centering
    \begin{subfigure}[b]{0.48\textwidth}
      \centering
      \includegraphics[width =0.9\linewidth]{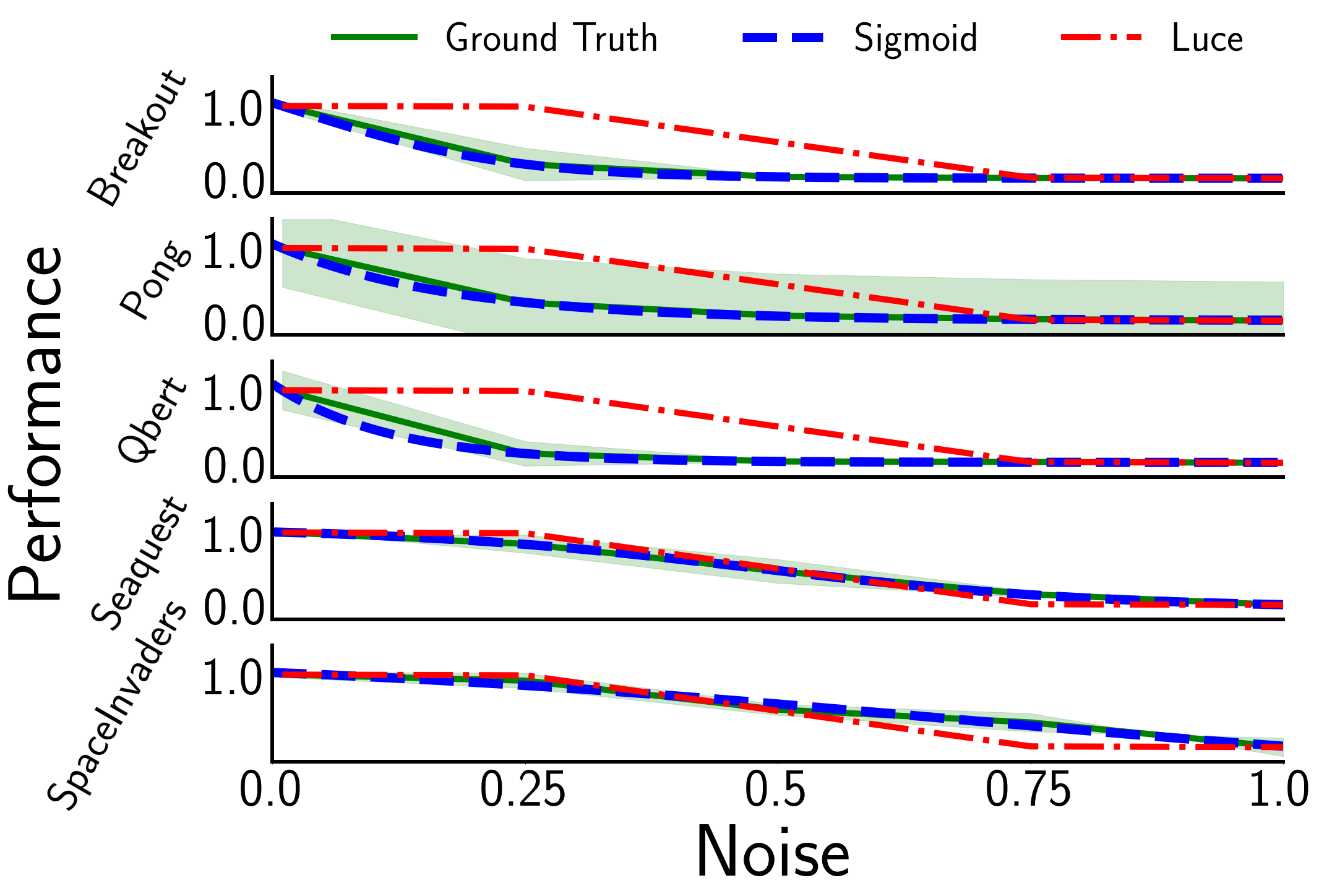}
      \caption{Atari Noise-Performance Relationship}
      \label{fig:atari_performance_noise}
    \end{subfigure}
    \begin{subfigure}[b]{0.48\textwidth}
      \centering
      \includegraphics[width =0.9\linewidth]{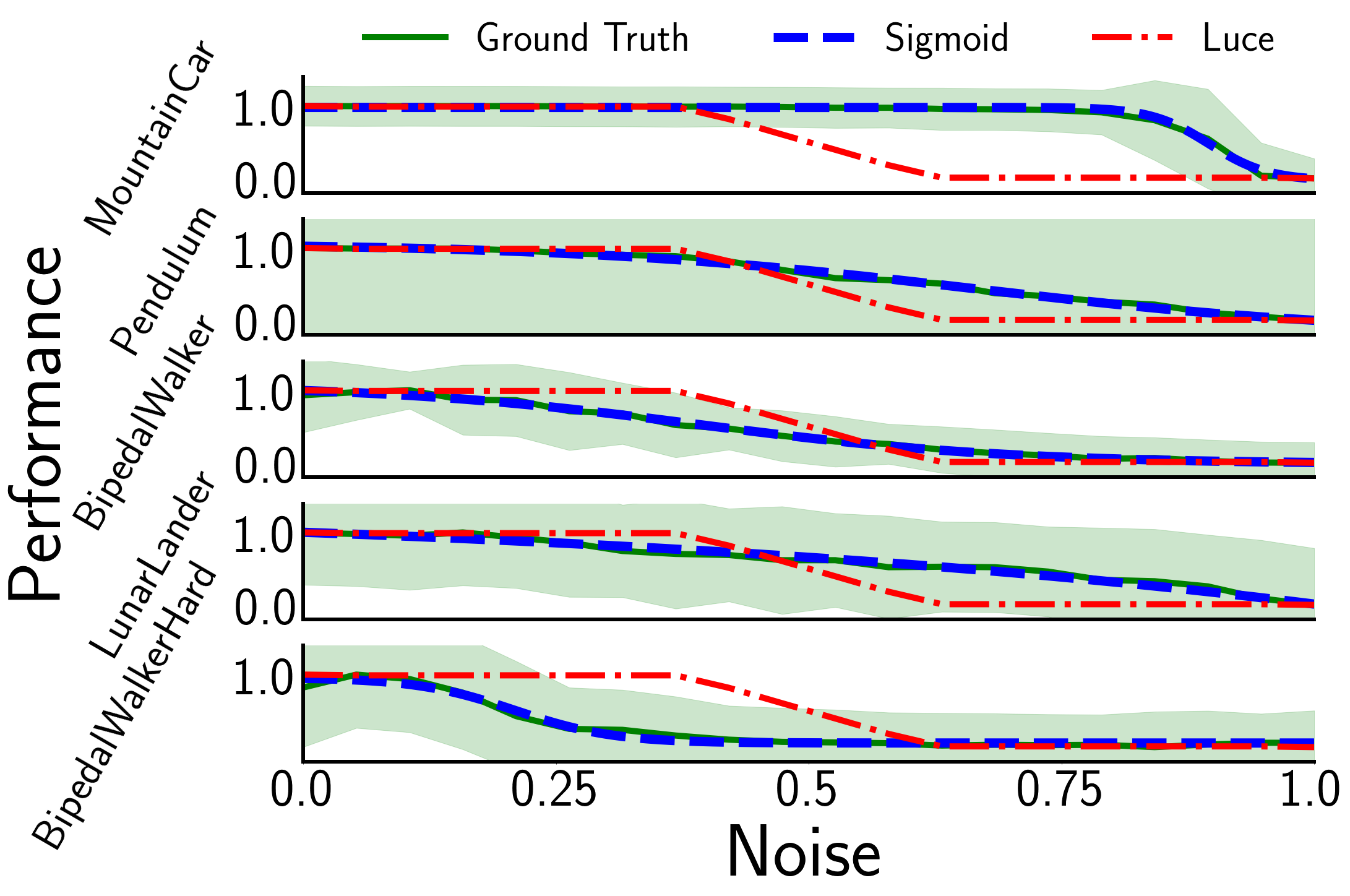}
      \caption{MuJoCo Noise-Performance Relationship}
      \label{fig:mujoco_performance_noise}
    \end{subfigure}
\caption{This figure depicts the noise-performance relationship in Atari and MuJoCo environments. The shaded regions for Ground Truth represent standard deviations over 20 trials. We utilize 5 and 20 noise levels equal-spaced between $0$ and $1$ for Atari and MuJoCo environments, respectively.}
\label{fig:performance_noise}
\end{figure}

We observe from Figure \ref{fig:atari_performance_noise} and \ref{fig:mujoco_performance_noise} that the noise-performance relationship is well-characterized by a four-parameter sigmoid function (Equation \ref{eq:fit}), where $c$ is a scaling factor, $y_0$ is the vertical offset, $x_0$ is the horizontal offset, and k modulates the sigmoid's slope.
\par\nobreak{ \small \noindent
\begin{align}
\label{eq:fit}
    \sigma(\eta)=\frac{c}{1+\exp(-k(\eta-x_0))}+y_0. 
\end{align}}To measure the accuracy of the sigmoid fit, we compute the $R^2$ value between the ground-truth reward and the learned reward function. With our sigmoid-based method, the average $R^2$ across all environments and policies is $M=0.996$ with  $SD=0.004$. In comparison, if we utilize the assumption of D-REX, the average $R^2$ is $M=0.812$ with $SD =0.153$.

Based on these observations, we propose to learn the noise-performance relationship from the data. However, learning this relationship is not trivial as we do not have access to the ground-truth reward function. Moreover, the relationship is dependent on the environment configurations (e.g., transition function) and the policy. In Section 4, we develop a novel approach for learning from suboptimal demonstrations that creates and leverages synthetic data to predict the noise-performance relationship, assuming a sigmoidal characteristic supported by results in Figure \ref{fig:performance_noise}. With this predicted relationship, we infer the idealized reward function and learn a better-than-demonstration policy. 

\section{Our Approach}

In this section, we present our novel IRL framework, Self-Supervised Reward Regression (SSRR), which accurately learns idealized reward functions and high-performing policies from suboptimal demonstration. SSRR sets a new state-of-the-art in learning from suboptimal demonstration by 1) learning the noise-performance relationship for policy degradation (Section \ref{subsec:SSRR}) and 2) training under a new methodology, Noisy-AIRL, to learn a noise-robust reward function (Section \ref{subsec:noisy_AIRL}).

\subsection{Self-Supervised Reward Regression}
\label{subsec:SSRR}

To predict the noise-performance relationship, we first need a reward function to criticize a policy and second, a policy that we can bootstrap to synthesize a data set of noisy trajectories.
To accomplish both steps, we utilize Adversarial IRL (AIRL), producing a reward function and policy from suboptimal demonstration.\footnote{Note: AIRL is not the only possible IRL method we could leverage; we choose AIRL as it is the state-of-the-art and empirically lends itself well to this task.} We utilize AIRL's learned reward function to criticize the collected noisy trajectories, providing the algorithm information about how the reward a policy receives degenerates as noise is injected into the policy's rollouts. 

However, being able to estimate the reward of a set of noisy policy rollouts is not enough. Rather, we must then learn a reward function that is aware of the amount of noise injected to the policy and accurately predicts the degree to which the policy's performance is degraded when this amount of noise is injected. Such a reward function would be environment-adaptative, policy-adaptative, and noise-level-adaptative instead of static as assumed as in prior work~\cite{brown2020better}.

To learn this reward function, we propose Self-Supervised Reward Regression (SSRR), a novel IRL method consisting of three phases: Phase 1) self-supervised data generation, Phase 2) Noise-Performance characterization, and Phase 3) Reward \& Policy learning. Figure \ref{fig:ssrr_diagram} depicts SSRR.

\begin{figure*}[t]
  \centering
  \includegraphics[width = \linewidth]{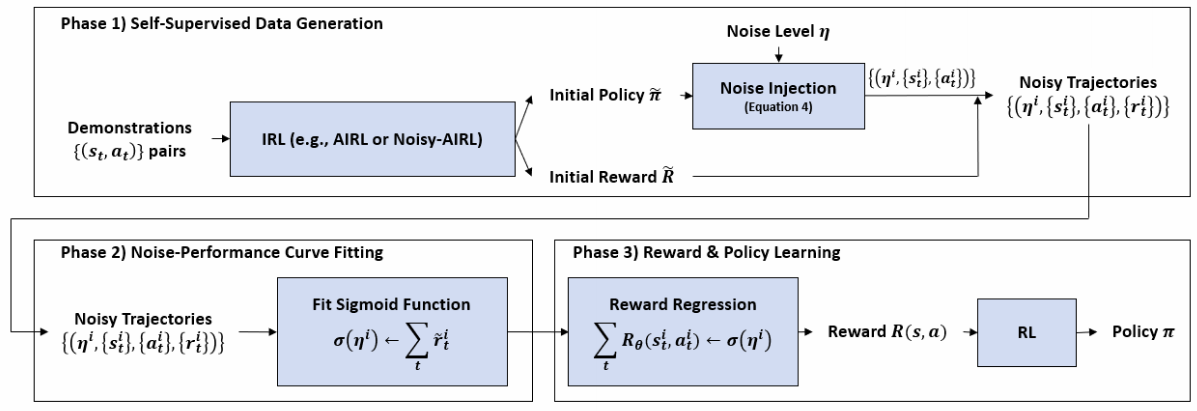}
  \vspace{-12pt}
  \caption{This figure depicts Self-Supervised Reward Regression (SSRR).}
  \label{fig:ssrr_diagram}
\end{figure*}

\textbf{Phase 1)}
For self-supervised data generation, we utilize the initial reward, $\Tilde{R}$, and initial policy, $\Tilde{\pi}$, from AIRL on suboptimal demonstration. Similar to D-REX, we inject uniform distribution noises into the learned policy (Equation \ref{eq:noise_injection}) to produce ``noisy'' policies. The trajectories, $\{\tau^i\}\ \forall i$, generated from these policies consist of four elements $\tau^i=\langle\eta^i, \{s_t^i\}, \{a_t^i\}, \{\Tilde{r}_t^i\}\rangle$: the noise parameter, $\eta^i$, a set of states and actions, $\{s_t^i\}$ and $\{a_t^i\}$, and the corresponding initial reward $\tilde{r}_t^i=\Tilde{R}(s_t^i,a_t^i)$. 

\textbf{Phase 2)}
For noise-performance curve fitting, we empirically characterize the noise-performance curve. Since we have access to an initial reward estimate for each trajectory, $\Tilde{R}(\tau^i)=\sum_t{\Tilde{R}(s_t^i,a_t^i)}$, we regress a sigmoid function towards the cumulative estimated rewards to determine the four sigmoid parameters in Equation \ref{eq:fit}. The sigmoid acts as a low-pass filter on the noise-performance relationship from AIRL's initial reward, creating a smooth noise-performance relationship even in the presence of high-frequency neural network output. 

\textbf{Phase 3)}
Leveraging the resultant noise-performance curve, we regress a reward function parameterized by trajectory states and actions, as shown in Equation \ref{eq:SSRR_loss}. Here, $\theta$ represents parameters of our idealized reward function. After obtaining an accurate reward function, $R_{\theta}$, we can apply RL to obtain a policy $\pi^*$ that outperforms the suboptimal demonstration, as demonstrated in Section \ref{sec:results}.
\par\nobreak{ \small \parskip0pt \noindent
\begin{align}
\label{eq:SSRR_loss}
    L_{\text{SSRR}}(\theta)=\mathbb{E}_{\tau^i}\left[\left(\left(\sum_{t=0}^T{R_\theta(s_t^i,a_t^i)}\right) - \sigma(\eta^i)\right)^2\right]
\end{align}}

\vspace{-8pt}
\subsection{An Improved Self-Supervised Data Generator: Noisy-AIRL}
\label{subsec:noisy_AIRL}
The learned reward function from AIRL (i.e., the discriminator) is biased towards evaluating the current policy, $\pi$, as shown in the expectations in Equation \ref{eqn:airl_loss}. Therefore, the estimated initial reward, $\Tilde{R}$, may not be reliable for trajectories drawn from a policy, $\pi'$, different than the current AIRL policy, $\pi$. This potential weakness is exposed by the need to learn a reward function that accurately discriminates between trajectories with varying noise levels and, in turn, varying optimality. 

Thus, we need a reward function that generalizes across the region of the state space occupied by noisy-policy trajectories (i.e., more robust to covariate shift in the region of noisy policies). Therefore, we propose a new method, Noisy-AIRL, that augments AIRL by injecting noise into the AIRL generator, $\pi$, to expose the reward function to a broader state space. Accordingly, we change the expectation over AIRL's discriminator objective to represent the noisy policy distribution and add an importance sampling adjustment into discriminator's loss, shown in Equation \ref{eq:state_dist}. 
\par\nobreak{ \small \noindent
\begin{align}
\label{eq:state_dist}
\begin{split}
    L_D=-\mathbb{E}_{\tau\sim\mathcal{D}, (s,a)\sim\tau}[\log D(s,a)]-\mathbb{E}_{\tau\sim\pi_\eta, (s,a)\sim\tau}\left[\frac{\pi(a|s)}{\pi_\eta(a|s)}\log (1-D(s,a))\right]
\end{split}
\end{align}}Prior work~\citep{pmlr-v78-laskey17a} has used noise injection to add robustness for IL. However, our work develops an IRL-based method that injects noise to learn both a robust reward function and a robust policy. We show in Section \ref{sec:results} that Noisy-AIRL learns more accurate initial reward estimates, $\Tilde{R}$, leading to a more accurate characterization of the noise-performance curve and higher-performing policies.


\begin{figure}
    \centering
    \begin{subfigure}[b]{0.48\textwidth}
      \centering
      \includegraphics[width =0.7\linewidth]{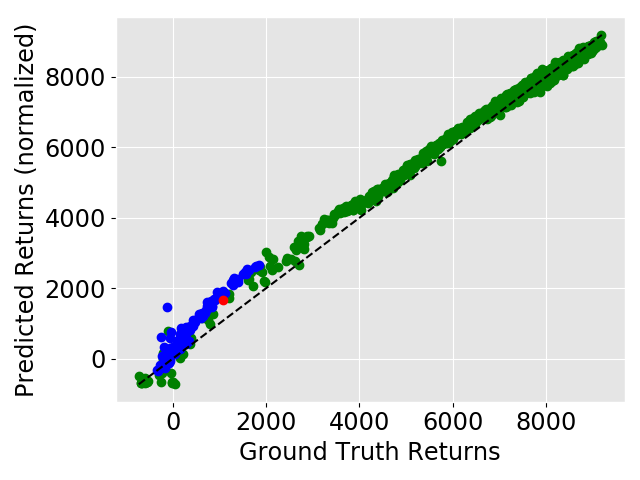}
      \caption{SSRR (Ours)}
      \label{fig:ssrr_halfcheetah_original}
    \end{subfigure}
    \begin{subfigure}[b]{0.48\textwidth}
      \centering
      \includegraphics[width =0.7\linewidth]{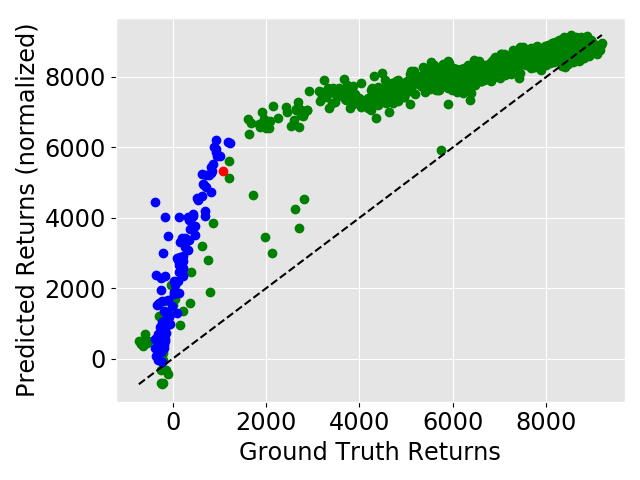}
      \caption{D-REX~\cite{brown2020better}}
      \label{fig:d_rex_halfcheetah_original}
    \end{subfigure}
\caption{Reward function correlation with Ground-Truth reward in HalfCheetah. Red dots represent demonstrations given to BC/AIRL, blue dots represent noise-injected trajectories (synthetic learning data), and green dots are unseen ``test'' trajectories used. The dotted line reflects perfect correlation. Both SSRR and D-REX returns are normalized to the range of ground-truth returns. }
\label{fig:reward_correlation_halfcheetah}
\end{figure}

\section{Results}
\label{sec:results}
In this section, we demonstrate the advantage of using SSRR across three MuJoCo virtual robotic control tasks~\citep{todorov2012mujoco}: HalfCheetah-v3, Hopper-v3, and Ant-v3. These robotic tasks are common benchmarks in RL and LfD literature, including D-REX~\citep{brown2020better}. 

We utilize two evaluation metrics in our comparison between SSRR and its closest benchmark, D-REX: 1) correlation of the learned reward function with ground-truth reward on both the training and testing data; 2) Ground-truth rewards of policies trained on a reward function learned from suboptimal demonstration. Together, these metrics evaluate the ability to recover ground-truth rewards (accuracy) and synthesize strong control policies (performance).

To construct a fair comparison between our approach and D-REX, we systematically vary whether the synthetic data set provided to each method is generated through BC (D-REX's method), AIRL, and Noisy-AIRL (our approach). We generate five data sets for each method to show repeatability. Additional details are provided in the supplementary material. 

\subsection{Reward Results}
\label{subsec:results_reward}
Table \ref{tab:SSRR_Corr} reports the correlation of the learned reward function with the ground-truth reward for SSRR and D-REX across three domains and three different synthetic data generation methods. Figure \ref{fig:reward_correlation_halfcheetah} depicts the correlation of each method in the HalfCheetah environment. We observe that SSRR recovers more accurate reward functions than D-REX across all domains and data-generation methods. A combination of our SSRR and Noisy-AIRL approaches achieves a correlation of $0.941$, $0.940$, and $0.970$ with the ground-truth reward versus only $0.845$, $-0.023$, and $0.621$\ for D-REX across the three simulated robot control tasks. 
These results show that SSRR recovers a more accurate representation of the latent reward function underlying the suboptimal demonstration. Moreover, these results demonstrate that Noisy-AIRL generates a better, synthetic data set of noisy trajectories for helping SSRR learn an accurate reward function. Comparing the ability to rank trajectories, SSRR also achieves higher ranking accuracy: on HalfCheetah, SSRR achieves an accuracy of $87.5\%$ vs. $79.4\%$ for D-REX; On Hopper, SSRR achieves $85.4\%$ vs. $80.2\%$ for D-REX; On Ant, SSRR achieves $88.0\%$ vs. only $68.13\%$ for D-REX. In contrast to the consistently high performance of our approach, D-REX's performance varies widely, achieving an average correlation of only $0.621$ on Ant but up to $0.845$ on HalfCheetah-v3. SSRR maintains superior performance by accurately capturing the domain-specific relationship between the performance of a learned policy and the amount of noise injection for synthetic data generation. We hypothesize D-REX works best with BC on Ant-v3 as BC-based noise-injection might better fit the Luce-Shepard rule. 

\begin{table}[t]
\caption{Learned Reward Correlation Coefficients with Ground-Truth Reward Comparison between SSRR and D-REX with different noisy trajectory generation methods in three domains. Reported results are Mean (Standard Deviation) from five trials. Note that SSRR + Noisy-AIRL is our method, and D-REX + BC is the approach of prior work~\cite{brown2020better}.}
\begin{center}
\begin{tabular}{P{50pt}P{60pt}P{40pt}P{40pt}P{40pt}P{40pt}P{40pt}}
\hline
\multirow{2}{*}{Domain} & \multicolumn{2}{c}{\textbf{Noisy-AIRL}} & \multicolumn{2}{c}{AIRL} & \multicolumn{2}{c}{BC} \\
\cline{2-7}& \textbf{SSRR} &  D-REX & SSRR &  D-REX & SSRR & D-REX \\
\hline
\multirow{2}{*}{HalfCheetah-v3} & \textbf{0.941} & 0.860 & 0.917 & 0.895 &  0.883 &  0.845 \\
&  \textbf{(0.025)} & (0.018) & (0.017) & (0.029) & (0.052) & (0.028)\\
\hline
\multirow{2}{*}{Hopper-v3} &  \textbf{0.940} & 0.880 & 0.813 & 0.797  & -0.050 & -0.023 \\
&  \textbf{(0.008)} & (0.017)& (0.031) & (0.015)  & (0.031) & (0.129) \\
\hline
\multirow{2}{*}{Ant-v3} & \textbf{0.970} & 0.199 & 0.615 & 0.335 & 0.649 & 0.621 \\
&  \textbf{(0.006)} & (0.010) & (0.024) & (0.045) & (0.071) & (0.024) \\
\hline
\label{tab:SSRR_Corr}
\end{tabular}
\end{center}
\vspace{-30pt}
\end{table}

\begin{table}[h]
\caption{Ground-Truth Reward of Final Policy Trained by Reward Function from SSRR and D-REX.}
\begin{center}
\begin{tabular}{cccccc}
\hline
\multirow{2}{*}{Domain} & \multicolumn{3}{c}{Demonstration} & SSRR & D-REX \\
\cline{2-6} & \#Demo & Average & Best & Average (Percentage) & Average (Percentage) \\
\hline
HalfCheetah-v3 & 1 & 1085 & 1085 & \textbf{2853 (263\%)} & 1619 (149\%) \\
Hopper-v3 & 4 & 1130 & 1253 & \textbf{3298 (292\%)} & 1574 (139\%) \\
Ant-v3 & 10 & 1157 & 1202 & \textbf{3944 (341\%)} & 2185 (189\%) \\
\hline
\label{tab:policy_trained}
\end{tabular}
\end{center}
\vspace{-30pt}
\end{table}

\subsection{Policy Results}
\label{subsec:results_policy}
We utilize Soft Actor-Critic (SAC~\citep{haarnoja2018soft}) for training a policy from SSRR's reward function. We measure the ground-truth reward of the policy trained from SSRR versus D-REX across three domains, as shown in Table \ref{tab:policy_trained}. For both D-REX and SSRR, we chose the reward function with the highest correlation with ground-truth from five repeats to train the policy. A policy trained with SSRR is able to not only able to outperform the suboptimal demonstration by $199\%$ on average, but also able to outperform D-REX by $87.5\%$. These results show that our approach almost doubles the performance advantage of prior work and sets a new state-of-the-art for learning from suboptimal demonstration. The key to our approach's success lies in better capturing the noise-performance relationship in SSRR and training a more robust reward function with data from Noisy-AIRL.


\section{Real-world Demonstration: Table-Tennis}
\label{sec:results_sawyer}
In our real-world experiment, we use a Rethink Robotics Sawyer to learn a topspin table tennis strike. There are three key aspects of successfully striking a topspin shot: 1) returning the ball over the net, 2) generating longitudinal speed, and 3) generating spin, resulting in a sharp drop of the ball after crossing the net. Through kinesthetic teaching, we provided five suboptimal demonstrations in the form of topspin strikes that more narrowly made it over the net, traveled more slowly, and had less topspin. Additional details about our setup and training process are included in the supplementary. A frame-by-frame display of our learned policy is in Figure \ref{fig:sawyer_ping_pong}. A video of the learned policy is included in the supplementary material. We quantitatively evaluate our learned policy by evaluating the longitudinal return speed of the ping pong ball, the acceleration caused due to spin, and the SSRR reward function. The acceleration in the vertical direction is a result of the net effects of gravity, drag, and the Magnus effect (i.e., topspin). Accordingly, our metric for evaluating the amount of topspin is the net, downward acceleration. A greater topspin results in a greater net downward acceleration.\\
We compare between the 5 suboptimal demonstrations, 5 AIRL policy trajectories, and 10 successful trajectories of SSRR policy in Table \ref{tab:ping_pong}. A one-way ANOVA and post-hoc test with DUNN-ŠIDÁK correction show that our SSRR policy has higher return speed ($p<.05$), more spin ($p<.05$), and higher reward ($p<.01$). More details are located in the supplementary. 

\begin{table}[t]
\caption{Performance of Table-Tennis SSRR Policy versus AIRL and Suboptimal Demonstrator.}
\begin{center}
\begin{tabular}{cccc}
\hline
\multirow{2}{*}{Metric} & Demonstration & AIRL & SSRR (Ours) \\
\cline{2-4}  & Average (Stdev) & Average (Stdev) & Average (Stdev) \\
\hline
Return Lateral Speed (pixels/sec) & 122.76 (36.74) & 107.13 (1.66) & \textbf{161.97 (31.19)}  \\
Vertical Acceleration (pixels/sec/sec) & 318.99 (6.56) & 331.04 (27.43) & \textbf{363.57 (29.37)} \\
SSRR Reward &  61.83 (4.24) & 73.73 (2.33) & \textbf{86.26 (13.45)}  \\
\hline
\label{tab:ping_pong}
\end{tabular}
\end{center}
\vspace{-8mm}
\end{table}

\begin{figure*}[h]
  \centering
  \includegraphics[width =\linewidth]{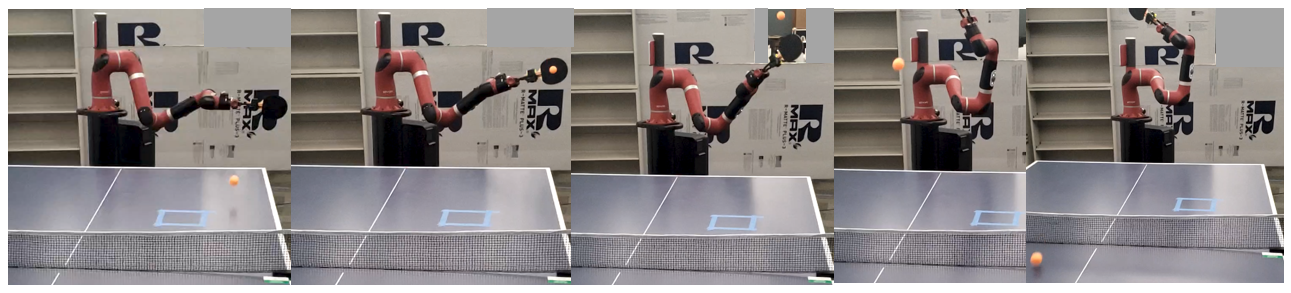}
  \vspace{-16pt}
  \caption{This figure depicts a frame-by-frame Sawyer Table Tennis Robot policy illustration.}
  \label{fig:sawyer_ping_pong}
  \vspace{-16pt}
\end{figure*}


\section{Related Work}
\label{sec:related_work}
IL and IRL techniques commonly assume optimality of demonstration~\citep{abbeel2004apprenticeship,Chen2020JointGA,Gombolay:2016a}. While some approaches slightly relax this assumption to stochastic optimality~\citep{ziebart2008maximum,ramachandran2007bayesian}, even this assumption does not hold for most real-world scenarios. \citet{kaiser1995obtaining} analyzes five sources of suboptimality within the human demonstration: unnecessary actions, incorrect actions, unmotivated actions, demonstration scenario is too limited for generalization, or specification of a wrong intention. \citet{nikolaidis2017game} and \citet{Paleja2019InterpretableAP} consider that when giving demonstrations, experts typically adopt diverse heuristics (i.e., ``mental shortcuts"~\citep{newell1972human}) to approximately solve optimization problems.

Researchers have attempted to learn from suboptimal demonstration by utilizing ranking information among trajectories as a supervision signal~\citet{dmirl}. While this procedure may provide accurate ranking information, this supervision is expensive and prone to error, especially given the continuous nature of robot trajectories. \citet{pmlr-v24-valko12a} takes a step in reducing the cost by collecting demonstrations and pre-labeling the data as expert or non-expert however, this approach loses granularity. \citet{Wu2019ImitationLF} presents an approach where unlabeled crowd-sourced data can be used to learn from suboptimal demonstration with the augmentation of confidence scores, which can be inaccurate, to assess the trajectory's optimality. \citet{brown2019extrapolating} instead queried humans to rank trajectories providing pairwise correlations, a process which is time-intensive and potentially noisy. Later work, D-REX~\citep{brown2020better}, automatically generated rankings by using noise-injected policies.

\textcolor{black}{Our demonstration of SSRR for robot table tennis was inspired by \citet{peters2013towards}, who proposed robot table tennis as a challenging task for robot learning. In a series of papers, these researchers designed sophisticated models, features~\cite{muelling2014learning}, and rewards~\cite{mulling2013learning} to train versatile policies for robot skill learning~\cite{muelling2010learning,KOC2018121}.}

\section{Conclusion}
\label{sec:conclusion}
In this paper, we develop a novel IRL algorithm, SSRR, to learn from suboptimal demonstration by characterizing the relationship between a policy's performance and the amount of injected noise. We introduce Noisy-AIRL to improve SSRR's robustness to covariate shift induced by noise injection. We empirically validate we can learn an idealized reward function from a suboptimal demonstration with $\sim0.95$ correlation with the ground-truth reward versus only $\sim 0.75$ for prior work. With this reward function, we can train policies achieving $\sim 200\%$ improvement over the suboptimal demonstration. Finally, we present an implementation of our algorithm on a real-world table tennis robot that learns to significantly exceed the ability of the human demonstrator ($p < 0.05$).


\clearpage
\acknowledgments{We to thank Dr.~Sonia Chernova for her mentorship and helpful critique of this research as well as that of Dr. Nakul Gopalan. We also wish to thank our reviewers for their valuable feedback in revising our manuscript. This work was sponsored by MIT Lincoln Laboratory grant 7000437192, the Office of Naval Research under grant N00014-19-1-2076, NASA
Early Career Fellowship grant 80HQTR19NOA01-19ECF-B1, and a gift to the Georgia Tech Foundation from Konica Minolta.}


\small
\bibliography{example}  
\clearpage
\appendix
\section{Real-world Demonstration: Table-Tennis}
In this section, we discuss our experimental setup of the real-world table tennis robot environment alongside details of our training framework and results.
\subsection{Experiment Setup}
The setup of our table tennis environment consists of a 7-degree of freedom Sawyer robot arm from Rethink Robotics, two Go-Pro Hero7 cameras, and a Newgy Robo-Pong 2055 ball feeder. The angle and speed the ball feeder was calculated empirically previous to collecting human demonstrations and kept constant throughout our experiment.

To communicate with Sawyer, we used Rethink Robotic's Intera SDK that integrates seamlessly with the Robot Operating System (ROS) middleware. We record the participant's demonstrations by subscribing to the robot joint state topic. This provides us with joint position and velocity at a rate of of 100 Hz.

For our vision system, we used the GoPro Hero7 Black to track the ping pong ball's trajectory throughout the demonstrations. We chose these cameras for their high frame rate. ElGato Camlinks were used as a connection bus to convert readily stream image data into the computer. This permits us a streaming rate of 60 Hz. To estimate the ping pong ball 3D position in real-world, we chose the triangulation method with a parallel axis camera set up of stereo vision. We utilized OpenCV library to detect the ball using carefully designed HSV values and contour size. We utilize ping pong ball transition dynamics alongside the calculated location of the ping pong ball in a Extended Kalman Filter to produce an accurate estimate of the ping pong ball's location throughout a trajectory.

\subsection{Implementation Details and Hyperparameters}
Noisy-AIRL's hyperparameters for table tennis could be found in Table \ref{tab:noisy_airl_hyperparameters_table_tennis}. 
\begin{table}[h]
\caption{Hyperparameters for AIRL and Noisy-AIRL in Table Tennis}
\begin{center}
\begin{tabular}{cc}
\hline
Hyperparameter & Value \\
\hline
discriminator\_update\_per\_step & 10 \\
max\_path\_length & 200 \\
episode\_per\_train\_step & 1 \\
$\gamma$ & 0.99 \\
GAE $\lambda$ & 0.95 \\
TRPO K-L step size & 0.001 \\
TRPO conjugate gradient steps & 10 \\
Train steps & 500 \\
\hline
\label{tab:noisy_airl_hyperparameters_table_tennis}
\end{tabular}
\end{center}
\end{table}

For noise injection, since we control the robots via speed mode, we cannot inject uniform-sampled action (speed) into the robot trajectory. Instead, we inject randomly generated Gaussian noises into each joints' speed command. We injected noises with 9 levels ($\eta\in\{0,1,2,3,4,5,6,7,8\}$). For each level, we generated $\eta$ number of Gaussian noises to inject into the original trajectory, ensuring a smooth movement. We observed clear performance degradation from noise level $0$ to noise level $8$. 

The SSRR hyperparameters for table tennis are the same as in the simulation environments. 

\subsection{Results}

\begin{figure}[tb]
    \centering
    \begin{subfigure}[b]{\textwidth}
      \centering
      \includegraphics[width =0.9\linewidth]{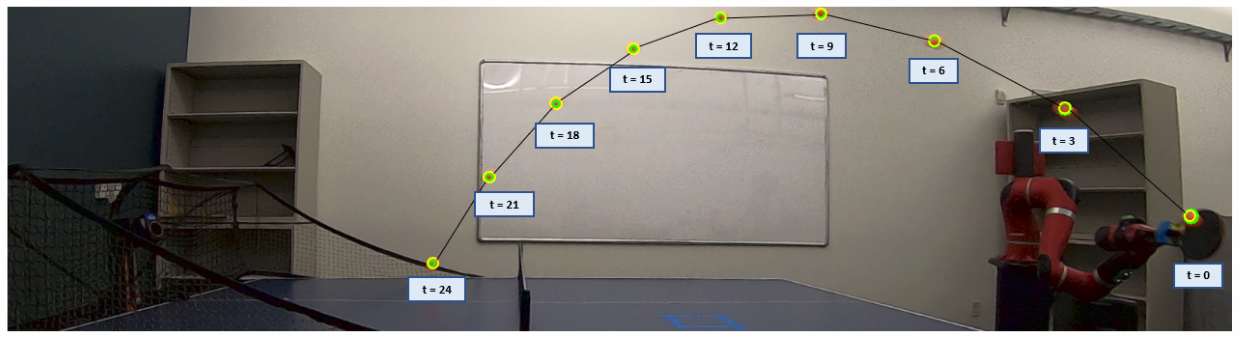}
      \caption{SSRR Policy}
      \label{fig:ssrr_spin}
    \end{subfigure}
    \begin{subfigure}[b]{\textwidth}
      \centering
      \includegraphics[width =0.9\linewidth]{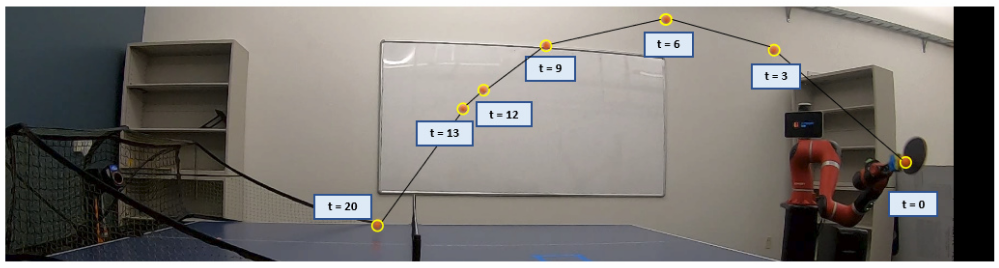}
      \caption{AIRL Policy}
      \label{fig:airl_spin}
    \end{subfigure}
    \begin{subfigure}[b]{\textwidth}
      \centering
      \includegraphics[width =0.9\linewidth]{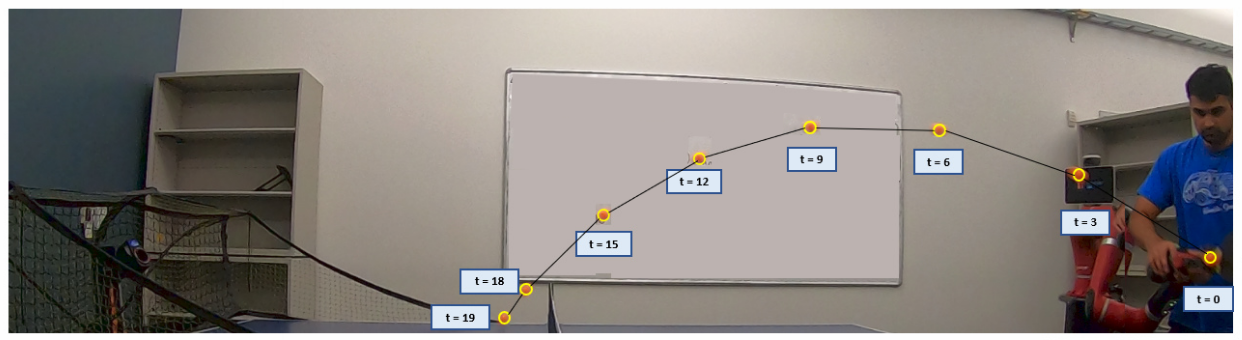}
      \caption{Demonstrator Suboptimal Policy}
      \label{fig:human_spin}
    \end{subfigure}
\caption{This figure depicts the arc of the returned ball. A sharper drop is correlated with higher spin. As seen, our policy is able to produce a superior topspin. The ping pong ball has been enhanced for viewing.}
\label{fig:sharp_drop}
\end{figure}

In our real-world experiment, we use a Rethink Robotics Sawyer to learn a topspin table tennis strike. There are three key aspects of successfully striking a topspin shot: 1) returning the ball over the net, 2) generating longitudinal speed, and 3) generating spin, resulting in a sharp drop of the ball after crossing the net. Through kinesthetic teaching, we provided five suboptimal demonstrations in the form of topspin strikes that more narrowly made it over the net, traveled more slowly, and had less topspin.  

We quantitatively evaluate our learned policy by evaluating the longitudinal return speed of the ping pong ball, the acceleration caused due to spin, and the SSRR reward function. The acceleration in the vertical direction is a result of the net effects of gravity, drag, and the Magnus effect (i.e., topspin). Our metric for evaluating the amount of topspin is the net, downward acceleration. A greater topspin results in a greater net downward acceleration.

We compare between the 5 suboptimal demonstrations, 5 AIRL policy trajectories, and 10 successful trajectories of SSRR policy in Table \ref{tab:ping_pong}. To verify that an SSRR policy has higher return speed, more spin, and higher reward, we utilize one-way ANOVA and posthoc test with DUNN-ŠIDÁK correction. We first test for normality and homoscedasticity and do not reject the null hypothesis in either case, using Shapiro-Wilk ($p > .05$) and Levene's Test ($p > .05$). Thus, we could perform one-way ANOVA and posthoc test with DUNN-ŠIDÁK correction for the three metrics, resulting in

\paragraph{Return Lateral Speed} One-way ANOVA shows significance ($F(17,2)=6.01, p<.05$). Posthoc test with DUNN-ŠIDÁK correction shows significant difference between SSRR with either demonstration ($p<.01$) or AIRL ($p<.05$), and no significant difference between AIRL and demonstration. 
\paragraph{Vertical Acceleration} One-way ANOVA shows significance ($F(17,2)=4.54, p<.05$). Posthoc test with DUNN-ŠIDÁK correction shows significant difference between SSRR with either demonstration ($p<.01$) or AIRL ($p<.05$), and no significant difference between AIRL and demonstration. 
\paragraph{SSRR Reward} One-way ANOVA shows significance ($F(17,2)=9.10, p<.01$). Posthoc test with DUNN-ŠIDÁK correction shows significant difference between SSRR with either demonstration ($p<.01$) or AIRL ($p<.05$), and significant difference between AIRL and demonstration ($p<.001$). 

As seen, our policy outperforms the average demonstration by a wide margin.
The success rate of our topspin strike is 75$\%$, compared with AIRL policy's success rate of $20\%$. The increase in spin can be visualized in Figure \ref{fig:sharp_drop}.

\clearpage

\section{Proof of Theorem 1}
\begin{theorem}
Consider three trajectories associated with three noise levels, $(\eta_1,\tau_1)$, $(\eta_2,\tau_2)$, and $(\eta_3,\tau_3)$, with $\eta_1<\eta_2<\eta_3$. Denote $r_i$ as the cumulative reward for $\tau_i$, i.e., $R_\theta(\tau_i)$. Optimizing D-REX's loss function (Equation \ref{eqn:D-REX}) results in $r_2=\frac{r_1+r_3}{2}$. 
\end{theorem}
\begin{proof}

We could view the problem as fixing $r_1$ and $r_3$, and optimize $r_2$ to minimize Equation \ref{eqn:D-REX}. 

Thus, we could write the loss
\begin{align*}
    L&=-\frac{1}{3}\left(\log\frac{\exp(r_1)}{\exp(r_1)+\exp(r_2)}+\log\frac{\exp(r_1)}{\exp(r_1)+\exp(r_3)}+\log\frac{\exp(r_2)}{\exp(r_2)+\exp(r_3)}\right) \\
\end{align*}
To optimize $r_2$, 
\begin{align*}
    r_2^*&=\arg\min_{r_2}L\\
    &=\arg\max_{r_2}\log\frac{\exp(r_1)}{\exp(r_1)+\exp(r_2)}+\log\frac{\exp(r_1)}{\exp(r_1)+\exp(r_3)}+\log\frac{\exp(r_2)}{\exp(r_2)+\exp(r_3)}\\ 
    &=\arg\max_{r_2}\log\frac{\exp(r_1)}{\exp(r_1)+\exp(r_2)}+\log\frac{\exp(r_2)}{\exp(r_2)+\exp(r_3)}\quad\left(\text{Middle-term not dependent on $r_2$}\right)\\
    &=\arg\max_{r_2}\log\frac{\exp(r_1)}{\exp(r_1)+\exp(r_2)}\frac{\exp(r_2)}{\exp(r_2)+\exp(r_3)} \\
    &=\arg\max_{r_2}\log\frac{\exp(r_1)}{\exp(r_2)+\frac{\exp(r_1)\exp(r_3)}{\exp(r_2)}+\exp(r_1)+\exp(r_3)}\quad\left(\text{Divide both by $\exp(r_2)$}\right)\\
\end{align*}
Because $\exp(r_2)+\frac{\exp(r_1)\exp(r_3)}{\exp(r_2)}\geq 2\sqrt{\exp(r_1)\exp(r_3)}$ and the equality holds if and only if  $\exp(r_2)=\sqrt{\exp(r_1)\exp(r_3)}$, 
\small
\begin{align}
\label{eq:proof_ineq}
    \log\frac{\exp(r_1)}{\exp(r_2)+\frac{\exp(r_1)\exp(r_3)}{\exp(r_2)}+\exp(r_1)+\exp(r_3)}\leq \log\frac{\exp(r_1)}{\exp(r_1)+\exp(r_3)+2\sqrt{\exp(r_1)\exp(r_3)}},
\end{align}
\normalsize
the equality in Equation \ref{eq:proof_ineq} only holds when $\exp(r_2)=\sqrt{\exp(r_1)\exp(r_3)}$. Therefore, the loss is minimized when $\exp(r_2)=\sqrt{\exp(r_1)\exp(r_3)}$, which is equivalent to $r_2=\frac{r_1+r_3}{2}$. 
\end{proof}

\clearpage

\section{Different Policy Noise-Performance Relationship}
We show the Noise-Performance relationship of five Atari games under three different RL methods trained RL agents in Figure \ref{fig:big_performance_noise_atari}. Specifically, for MDP and each policy, we first train the policy to convergence in the MDP, and then inject five levels of noise ($0.0, 0.25, 0.50, 0.75, 1.0$) to generate one sub-plot. We show the Noise-Performance relationship of five MuJoCo environments with agents trained under three different RL approaches in Figure \ref{fig:big_performance_noise_mujoco}. We inject 20 noise levels in MuJoCo environments (equal interval noises between $0$ and $1$). 

\begin{figure*}[h]
  \centering
  \includegraphics[width =\linewidth]{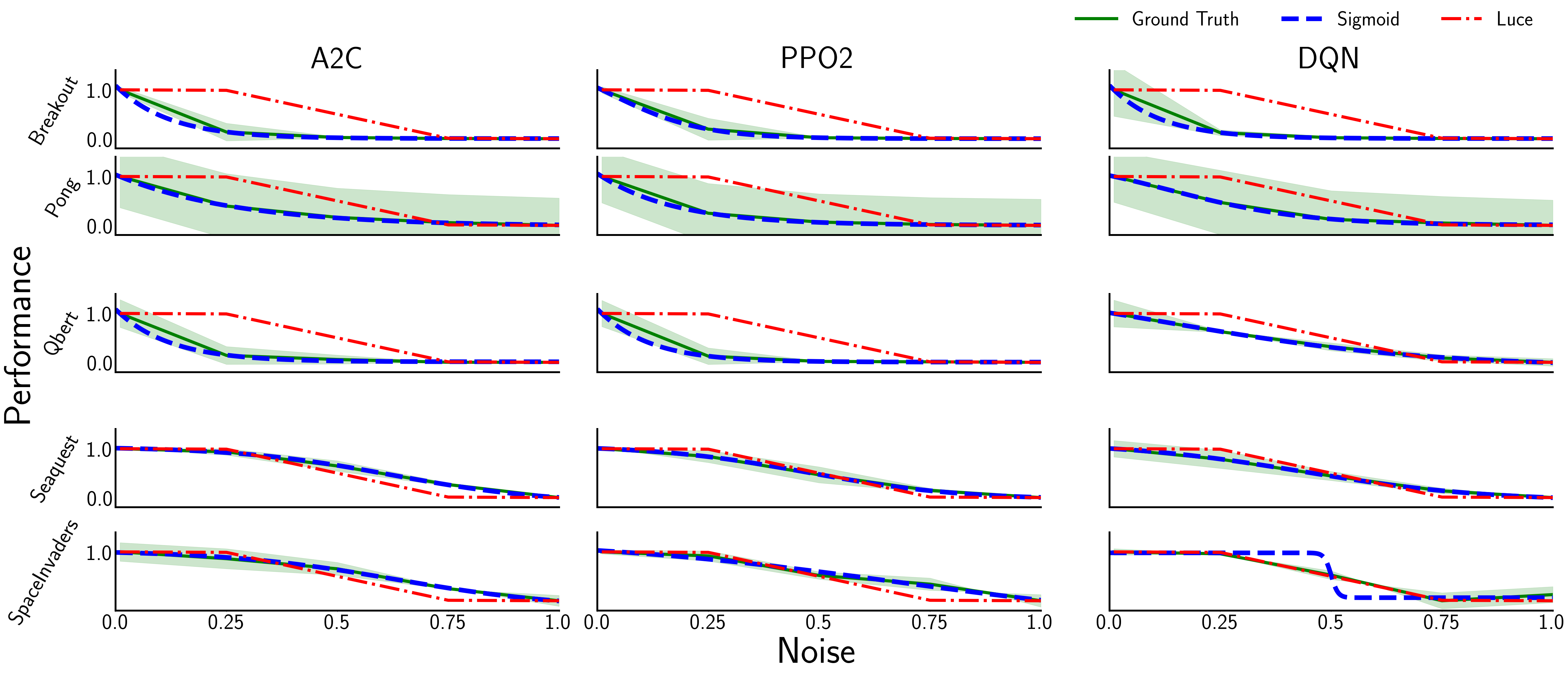}
  \caption{Noise-Performance Relationship in five Atari Environments across three different RL methods trained RL agents. }
  \label{fig:big_performance_noise_atari}
\end{figure*}

\begin{figure*}[h]
  \centering
  \includegraphics[width =\linewidth]{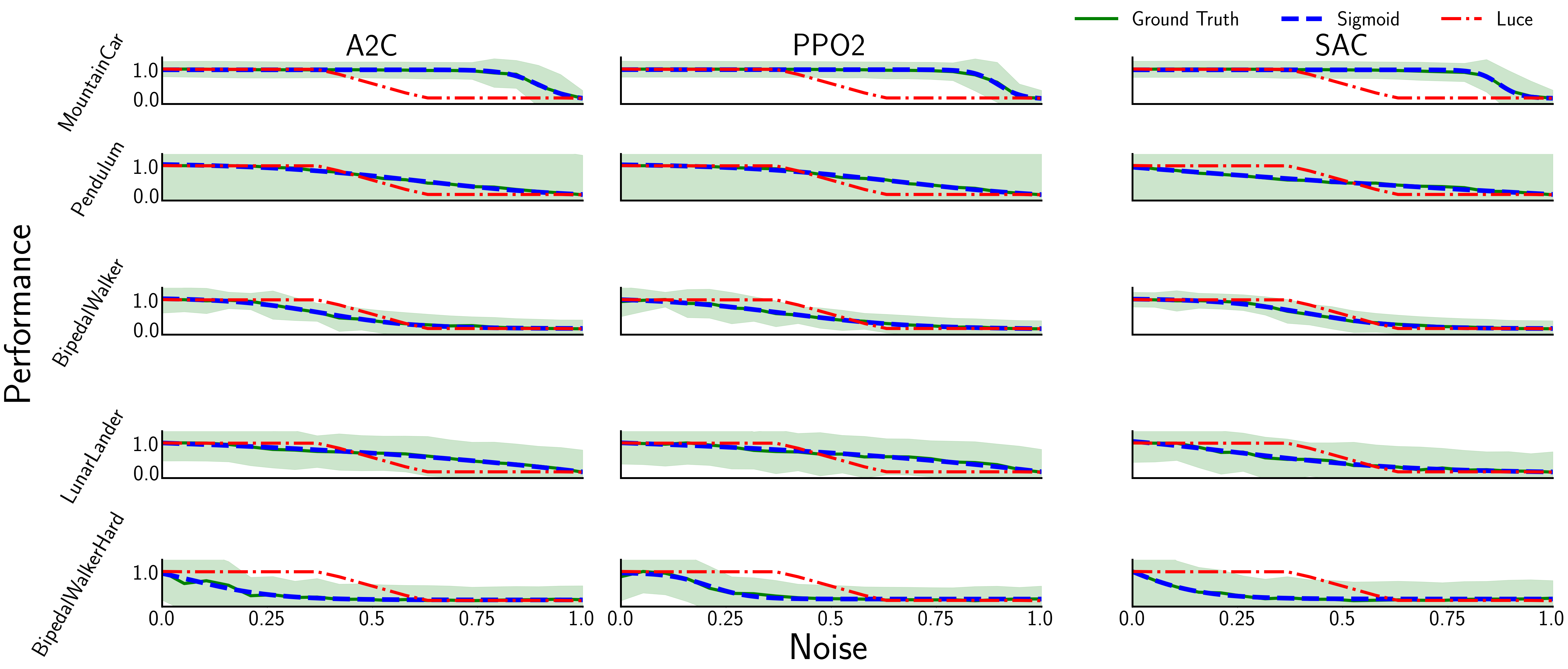}
  \caption{Noise-Performance Relationship in five MuJoCo Environments across three different RL methods trained RL agents. }
  \label{fig:big_performance_noise_mujoco}
\end{figure*}

\clearpage

\section{Simulation Experiment Details}
\subsection{Environment Details}
We note that we modified Hopper-v3 and Ant-v3 environments to disable the option ``terminate\_when\_unhealthy''. If the trajectory terminates once the agent is in an ``unhealthy'' state, the cumulative reward of the trajectory will almost be a linear function of the length of the trajectory, which means any positive constant reward (such as $+1$) could produce a near perfect reward correlation with ground-truth reward. Intuitively, if we give a $+1$ reward for Hopper and enable ``terminate\_when\_unhealthy'', it would learn to hop forward to stay alive because the hopper cannot stand still or hop backwards, which is unintended shortcut for the algorithm and does not produce useful information when comparing algorithms' effectiveness. We disabled  ``terminate\_when\_unhealthy'' for Hopper-v3 and Ant-v3 in both policy training and evaluation for all methods. 

\subsection{Experiment Details}
We identified that both algorithms' performances are highly related to the generated noisy trajectories $\tau^i$, i.e., different random seeds generate different noisy trajectories. In order to make an apple-to-apple comparison between our approach and D-REX, we generated noisy trajectory datasets from three different approaches (BC, AIRL, Noisy-AIRL). For each approach, we generated five noisy datasets, and experimented SSRR and D-REX on the same set of 5 datasets. 

\subsection{Snippet}
In the implementation of D-REX, authors also introduced ``snippets'', which are subsampled variant-length consecutive segments from full trajectories, i.e., 
\begin{align}
\mathbf{s}^k_{i,j}=(s^k_i,a^k_i,s^k_{i+1},a^k_{i+1}, \cdots, s^k_j,a^k_j), 1\leq i< j\leq \text{len}(\tau_k), i,j\in\mathbb{N},
\end{align}
in which $\text{len}$ is the length operator. 
The ranking pair dataset is based on same-length snippets instead of full trajectories to increase the data amount, i.e., 
\begin{align}
\text{for }(i,j):\tau_i\succ\tau_j, \text{we have } \mathbf{s}^i_{i_1,j_1}\succ\mathbf{s}^j_{i_2,j_2}.    
\end{align}
Similarly, we could subsample $\tau_k$ to be snippets and the corresponding cumulative reward learning target should be adjusted by the ratio of length: $\frac{\text{len}(\text{snippet})}{\text{len}(\tau^k)}$, and therefore, 
\begin{align}
    \sum_{t=i}^j{R_\theta(s_t^k, a_t^k)}\leftarrow \frac{\text{len}(\textbf{s}^k_{i,j})}{\text{len}(\tau_k)}\sigma(\eta^k).
\end{align}

\subsection{Implementation Details and Hyperparameters}
We utilized stable-baselines \citep{stable-baselines} implementation of Soft Actor-Critic (SAC \citep{haarnoja2018soft}) to obtain a policy for each environment. We then take partially trained policy to generate suboptimal demonstration, specifically step $36000$, $90000$, and $650000$ for HalfCheetah-v3, Hopper-v3, and Ant-v3, respectively. We also utilized the partially trained SAC agent to generate the unseen test data for reward correlation. For every $1000$ training step of SAC, we saved the agent and generated one trajectory to work as the unseen test data, and thus creating a spectrum of low-to-high performance trajectories. 

For Behavior Cloning (BC) and D-REX, we used D-REX's implementation (\url{https://github.com/dsbrown1331/CoRL2019-DREX}). For AIRL and Noisy-AIRL, hyperparameters are shown in Table \ref{tab:airl_hyperparameters_simulation}. We used noise levels of $[0.0, 0.0, 0.0, 0.0, 0.0, 0.05, 0.10, 0.15, 0.20, 0.25]$ for the ten trajectories within one train iteration. 
\begin{table}[h]
\caption{Hyperparameters for AIRL and Noisy-AIRL in Simulation Domains}
\begin{center}
\begin{tabular}{ccc}
\hline
Hyperparameter & Environment & Value \\
\hline
discriminator\_update\_per\_step & All & 10 \\
max\_path\_length & All & 1000 \\
episode\_per\_train\_step & All & 10 \\
$\gamma$ & All & 0.99 \\
GAE $\lambda$ & HalfCheetah, Hopper & 0.95 \\
GAE $\lambda$ & Ant & 0.97 \\
TRPO K-L step size & HalfCheetah, Ant & 0.01 \\
TRPO K-L step size & Hopper & 0.005 \\
TRPO conjugate gradient steps & HalfCheetah, Hopper & 15 \\
TRPO conjugate gradient steps & Ant & 10 \\
Train steps & HalfCheetah, Hopper & 1000 \\
Train steps & Ant & 2000 \\
\hline
\label{tab:airl_hyperparameters_simulation}
\end{tabular}
\end{center}
\end{table}

For SSRR, we utilized snippets length of $50-500$ in our implementation (trajectory length is $1000$) and L2 regularization weight of $0.1$. 

\end{document}